\documentclass{article}

\usepackage{PRIMEarxiv}

\usepackage[utf8]{inputenc} 
\usepackage[T1]{fontenc}    
\usepackage{hyperref}       
\usepackage{url}            
\usepackage{booktabs}       
\usepackage{amsfonts}       
\usepackage{nicefrac}       
\usepackage{microtype}      
\usepackage{lipsum}
\usepackage{fancyhdr}       
\usepackage{graphicx}       
\graphicspath{{media/}}     
\usepackage{xcolor}         
\usepackage{mathtools}

\usepackage{times}  
\usepackage{helvet}  
\usepackage{courier}  
\usepackage[round]{natbib}  
\usepackage{caption} 
\DeclareCaptionStyle{ruled}{labelfont=normalfont,labelsep=colon,strut=off} 
\usepackage{algorithm}
\usepackage{algorithmic}
\usepackage{latexsym,amsmath,amssymb,bm}
\usepackage[font=footnotesize, skip=5pt]{caption}
\usepackage{times}
\usepackage{amsthm}
\usepackage[retainorgcmds]{IEEEtrantools}
\usepackage{mdwlist}
\usepackage{ctable}
\usepackage{enumitem}
\usepackage{listings}
\usepackage{thmtools,thm-restate}
\usepackage{multirow}
\usepackage{multicol}

\newtheorem{theorem}{Theorem}
\newtheorem{defn}{Definition}

\newtheorem{lma}{Lemma}

\newtheorem{cor}{Corollary}

\DeclareMathOperator*{\argmin}{arg\,min}
\DeclareSymbolFont{bbold}{U}{bbold}{m}{n}
\DeclareSymbolFontAlphabet{\mathbbold}{bbold}
\DeclarePairedDelimiter{\ceil}{\lceil}{\rceil}

\title{Sharp Analysis of Random Fourier Features in Classification
}

\author{
  Zhu Li \\
  Gatsby Computational Neuroscience Unit,  \\
  University College London, \\
  United Kingdom\\
  \texttt{zhu.li@ucl.ac.uk}
}

\begin{document}
\maketitle

\begin{abstract}
We study the theoretical properties of random Fourier features classification with Lipschitz continuous loss functions such as support vector machine and logistic regression. Utilizing the regularity condition, we show for the first time that random Fourier features classification can achieve $O(1/\sqrt{n})$ learning rate with only $\Omega(\sqrt{n} \log n)$ features, as opposed to $\Omega(n)$ features suggested by previous results. Our study covers the standard feature sampling method for which we reduce the number of features required, as well as a problem-dependent sampling method which further reduces the number of features while still keeping the optimal generalization property. Moreover, we prove that the random Fourier features classification can obtain a fast $O(1/n)$ learning rate for both sampling schemes under Massart's low noise assumption. Our results demonstrate the potential effectiveness of random Fourier features approximation in reducing the computational complexity (roughly from $O(n^3)$ in time and $O(n^2)$ in space to $O(n^2)$ and $O(n\sqrt{n})$ respectively) without having to trade-off the statistical prediction accuracy. In addition, the achieved trade-off in our analysis is at least the same as the optimal results in the literature under the worst case scenario and significantly improves the optimal results under benign regularity conditions.
\end{abstract}


\section{Introduction}
Kernel methods have been widely used in many machine learning tasks such as regression and classification \citep{Scholkopf01,scholkopf2004kernel}, as they provide a simple framework to model highly complicated functional relationships and well-established theoretical guarantees \citep{caponnetto2007optimal,steinwart2008support}. The power of kernel methods comes from the so-called "kernel trick", where it utilizes a feature function $\phi(\cdot)$ to implicitly map the data into a high or possibly infinite dimensional feature space and thus allows non-linear functional learning. However, kernel methods are notorious for being time-consuming, since a typical kernel learning algorithm requires $O(n^3)$ computation and $O(n^2)$ memory, where $n$ is the number of training samples. Due to the prohibitive computational requirements, a flurry of research has been devoted to developing algorithms that efficiently approximate kernel functions \citep{smola2000sparse,williams2001using,rahimi2007random,mahoney2009cur,alaoui2015fast,rudi2017falkon,zhang2015divide}.

Among many approximation frameworks, the random Fourier features (RFFs) method proposed by \citet{rahimi2007random} has received great attention recently~\citep[see][for a comprehensive review]{liu2020random}. The key idea of RFFs is to approximate the infinite dimensional feature map $\phi(\cdot)$ with an explicit $s$-dimensional random feature map $\phi_s(\cdot)$ through Bochner's theorem \citep{Bochner32,rudin2017fourier}, which states that $\phi_s(\cdot)$ can be constructed through sampling from some spectral measure. Kernel methods are now reduced to linear learning in the feature space, which can be computed via fast linear solver \citep{shalev2011pegasos}. The computational cost decreases from roughly $O(n^3)$ in time and $O(n^2)$ in space to $O(ns^2)$ and $O(ns)$ respectively. As a result, significant computational savings can be achieved as long as $s \ll n$.

Despite their empirical success \citep{rahimi2007random,huang2014kernel,dai2014scalable}, theoretical understanding of the RFFs is incomplete. In particular, the question of how to choose $s$ in order to obtain the RFFs estimators with performance provably comparable to original kernel methods remains unclear. To this end, several authors study the properties of the RFFs to approximate the kernel function and the kernel Gram-matrix~\citep[see e.g.,][and references therein]{rahimi2007random,sriperumbudur2015optimal,sutherland2015error}. However, all of these works require $s = \Omega(n)$ features to guarantee no loss of prediction accuracy, which translates to no computational savings at all. A highly refined analysis in the context of ridge regression is proposed recently~\citep[see e.g.,][]{rudi2017generalization,avron2017random,li2021towards}. When the spectral measure is used for sampling, they first show that $O(\sqrt{n}\log n)$ features are adequate to guarantee the minimax optimal learning rate $O(1/\sqrt{n})$, the same learning rate obtained with full kernel ridge regression. Furthermore, they prove that the RFFs regression can obtain a fast learning rate at the expense of increasing the number of features. Finally, they demonstrate that using a problem-dependent sampling distribution can significantly reduce the number of features to $s = \Omega(1)$. 

A question motivating our study is whether similar theoretical results hold in the classification setting where a key difference is the loss functions employed (Lipschitz continuous loss such as support vector machine and logistic regression). \citet{bach2017equivalence}, \citet{sun2018but}, and~\citet{li2021towards} study the generalization properties of RFFs approximations in the classification setting. They show that RFFs estimators can provide computational gains while still preserving the statistical properties of the original kernel method. Nevertheless, a key requirement in these analyses is to employ a certain problem-dependent sampling distribution. Computing such a distribution often requires $O(n^3)$ in time and $O(n^2)$ in space already and hence is itself intractable. Therefore, whether RFFs classification can provide computational savings without using the problem-dependent sampling distribution remains unclear, and a detailed trade-off between the number of features required and the statistical prediction accuracy is still missing.

A key step in obtaining a better trade-off for RFFs regression is to employ the regularity condition (see Assumption A.$3$). However, this property is not used while analyzing the RFFs classification. In this paper, by incorporating the regularity condition, we improve the optimal results in the literature and provide a definitive answer to questions mentioned above by making the following contributions
\begin{itemize}
    \item Under suitable regularity condition (Assumption A.$3$), Theorem~\ref{theo:minx} shows that RFFs classification only requires $\Omega(\sqrt{n}\log n)$ features to guarantee the minimax optimal $O(1/\sqrt{n})$ learning rate, the same prediction accuracy as the original kernel classification methods. Our analysis allows the computational cost to reduce from $O(n^3)$ in time and $O(n^2)$ in space to $O(n^2)$ and $O(n\sqrt{n})$ respectively, and suggest that for a wide range of classification problems, RFFs approximations provide dramatic computational cost savings without loss of prediction accuracy. To the best of our knowledge, this is the first result confirming that such a computational gain is possible in the classification setting when the standard sampling method is used.
    
    \item Using Massart's low noise assumption (Assumption A.$4$), Theorem~\ref{theo:fast_rate} further provides a more refined analysis on the generalization properties of the RFF classification estimators. We obtain a sharp $O(1/n)$ learning rate for classification at the expense of more random features required.

    \item We also discuss how problem-dependent sampling distribution further reduces the computational cost in the $O(1/\sqrt{n})$ rate setting and the $O(1/n)$ rate setting. Our analysis expresses the trade-off between the number of features required and the statistical prediction accuracy in terms of the regularization parameter ($\lambda$) and the \emph{effective degree of freedom} ($d(\lambda)$) and points out how utilizing the optimized feature can lead to a significant reduction in the computational cost.
    
    \item Finally, in Table~\ref{tab:wor_lip_pla} and~\ref{tab:ref_lip_wei}, we provide a comprehensive comparison between achieved results in this paper and the optimal bound in the literature. The analysis demonstrates that under benign conditions, our study obtains the sharpest bound on the number of features required in literature, while under worst case scenario, we match the optimal results in the literature.
\end{itemize}

\section{Background}
\subsection{Supervised Learning with Kernels} \label{sec:supervised_learning}
Let $P(x,y)=P_xP(y \mid x)$ be a joint probability density function defined on  $\mathcal{X}\times \mathcal{Y}$ where $\mathcal{X}$ is an instance space and $\mathcal{Y}$ a label space. While in regression tasks $\mathcal{Y} \subset \mathbb{R}$, in classification tasks it is typically the case that $\mathcal{Y}=\{-1, 1\}$. Let $\{(x_i,y_i)\}_{i=1}^n$  be a training set sampled independently from $P(x,y)$. The goal of a supervised learning defined with a kernel function $k$ (and the associated reproducing kernel Hilbert space $\mathcal{H}$) is to find a hypothesis $f \colon \mathcal{X} \rightarrow \mathcal{Y}$ such that $f \in \mathcal{H}$ and $f(x)$ is a good estimate of the label $y \in \mathcal{Y}$ corresponding to a previously unseen instance $x \in \mathcal{X}$. In particular, the learning can be formulated as the following optimization problem
\vspace{-.5em}
\begin{IEEEeqnarray}{rCl}
	\hat{f}^{\lambda}  &\coloneqq& \argmin_{f \in\mathcal{H}}\  \frac{1}{n}\sum_{i=1}^n l(y_i,f(x_i))+ \lambda \|f\|_{\mathcal{H}}^2\ . \nonumber 
\end{IEEEeqnarray}
where $l:\mathcal{Y}\times\mathcal{Y}\rightarrow \mathbb{R}_{+}$ is a loss function and $\lambda$ is the regularization parameter to prevent overfitting. As a result of the representer theorem~\citep{Scholkopf01}, an empirical risk minimization estimator in this setting can be expressed as $\hat{f}^{\lambda} =\sum_{i=1}^n \alpha_ik(x_i,\cdot)$ with $\alpha \in \mathbb{R}^n$ and the optimization problem can be reformulated as

\begin{IEEEeqnarray}{rCl}
\hat{\alpha}_k^{\lambda}  &\coloneqq& \argmin_{\alpha \in \mathbb{R}^n} \ \frac{1}{n}\sum_{i=1}^n l(y_i,(\mathbf{K}\alpha)_i) + \lambda \alpha^T \mathbf{K}\alpha \ , \label{eq:krl_opm}
\end{IEEEeqnarray}

where $\mathbf{K}$ is the kernel Gram-matrix with $\mathbf{K}_{i,j} = k(x_i,x_j)$.

\paragraph{Learning Risk} The hypothesis $\hat{f}^{\lambda}$ is an empirical estimator and we use the learning risk to assess its ability to capture the relationship between instances and labels given by $P$~\citep{caponnetto2007optimal}
\[\mathbb{E}_{P}[l_{\hat{f}^{\lambda}}] = \int_{\mathcal{X}\times \mathcal{Y}} l(y,\hat{f}^{\lambda}(x))dP(x,y)\ ,\]
where we use $l_{f}$ to denote $l(y,f(x))$. When the context is clear, we will omit $P$ from the expectation and write $\mathbb{E}[l_{\hat{f}^{\lambda}}]$.

The empirical distribution $P_n(x,y)$ is given by a sample of $n$ examples drawn independently from $P(x,y)$. The empirical risk is used to estimate the learning risk $\mathbb{E}[l_{\hat{f}^{\lambda}}]$ and it is given by \[\vspace{-.1em} \mathbb{E}_n[l_{\hat{f}^{\lambda}}] = \frac{1}{n} \sum_{i=1}^n l(y_i,\hat{f}^{\lambda}(x_i))\ .\] Similar to \citet{rudi2017generalization} and \citet{caponnetto2007optimal}, we will assume~\footnote{The existence of $f_{\mathcal{H}}$ depends on the complexity of $\mathcal{H}$ which is related to the data distribution $P(y|x)$. For more details, please see \citet{caponnetto2007optimal} and \citet{rudi2017generalization}.} the existence of $f_\mathcal{H} \in \mathcal{H}$ such that $f_{\mathcal{H}} = \argmin_{f \in \mathcal{H}}\ \mathbb{E}[l_f]$. 
Note that $\mathbb{E}[l_{f_{\mathcal{H}}}]$ is the lowest learning risk one can achieve in the reproducing kernel Hilbert space $\mathcal{H}$. Hence, theoretical studies of the estimator $\hat{f}^{\lambda}$ often concern how fast its learning risk $\mathbb{E}[l_{\hat{f}^{\lambda}}]$ converges to $\mathbb{E}[l_{f_{\mathcal{H}}}]$, that is, how fast the excess risk $\mathbb{E}[l_{\hat{f}^{\lambda}}] - \mathbb{E}[l_{f_{\mathcal{H}}}]$ converges to zero. In the remainder of the manuscript, we will refer to the rate at which the excess risk converges to zero as the \emph{learning rate}.

\subsection{Random Fourier Features}
Despite providing a flexible non-linear approximation framework, kernel methods suffer from the scalability issue. In particular, kernel supervised learning often requires the store or the inverse of the kernel Gram matrix $\mathbf{K}$ ($O(n^2)$ and $O(n^3)$ computations respectively), which is prohibitive. As a result, many low-rank approximation algorithms have been designed to resolve this issue~\citep[see, e.g.,][and references therein]{smola2000sparse,Williams01,rahimi2007random,rahimi2009weighted,mahoney2009cur}.

Among them, RFFs method is a widely used, simple, and effective technique for scaling up kernel methods. The idea is due to Bochner's theorem~\citep{Bochner32}, which states that any bounded, continuous, and shift-invariant kernel is the Fourier transform of
a bounded positive measure. Assuming the spectral measure $d\tau$ has a density function $p(\cdot)$, we can write the corresponding kernel as  
\begin{IEEEeqnarray}{rCl}
k(x,y) &=&  \int_{\mathcal{V}}e^{-2\pi iv^{T}(x-y)}d\tau(v) \nonumber \\
&=& \int_{\mathcal{V}} \big(e^{-2\pi i v^{T}x}\big)\big(e^{-2\pi i v^{T}y}\big)^{*}p(v)dv \ , \nonumber
\end{IEEEeqnarray}
where $c^{*}$ denotes the complex conjugate of $c \in \mathbb{C}$. Typically, the kernel is real valued and we can ignore the imaginary part~\citep[see e.g.,][]{rahimi2007random}. \citet{bach2017equivalence} and \citet{rudi2017generalization} further generalize the idea by considering the following decomposition of kernel functions
\begin{IEEEeqnarray}{rCl}
k(x,y) = \int_{\mathcal{V}}\psi(v,x)\psi(v,y)p(v)dv \ ,\label{krl_dec}
\end{IEEEeqnarray}
where $\psi \colon \mathcal{V}\times \mathcal{X}\rightarrow \mathbb{R}$ is a continuous and bounded function with respect to $v$ and $x$. Hence, we can approximate the kernel function using its Monte-Carlo estimate
\begin{IEEEeqnarray}{rCl}
\tilde{k}(x,y) &= & \frac{1}{s}\sum_{i=1}^s\psi(v_i,x)\psi(v_i,y)\ ,\nonumber\\
& =& \phi_s(x)^T\phi_s(y) \ .\label{krl_appx}
\end{IEEEeqnarray}
where $\{v_i\}_{i=1}^s$ are sampled independently from the spectral measure $p(v)$ and \vspace{-.5em} \[\phi_s(x) = \frac{1}{\sqrt{s}}[\psi(v_1,x),\dots,\psi(v_s,x)]^{\top}.\vspace{-.5em}\] We denote the reproducing kernel Hilbert space spanned by $\tilde{k}$ as $\tilde{\mathcal{H}}$ (note that in general $\tilde{\mathcal{H}} \nsubseteq \mathcal{H}$). Let $\tilde{\mathbf{K}}$ be Gram-matrices with entries $\tilde{\mathbf{K}}_{ij} = \tilde{k}(x_i,x_j)$. Then the following equalities can be derived easily from Eq.~(\ref{krl_appx})
\begin{align*}
\begin{aligned}
k(x,y) = \mathbb{E}_{v \sim p}\big[\tilde{k}(x,y) \big]\quad \wedge \quad \mathbf{K} = \mathbb{E}_{v\sim p}[\tilde{\mathbf{K}}] \ .
\end{aligned}
\end{align*}

In addition to the kernel Gram-matrix approximation, \citet{bach2017equivalence} establishes that any $f \in \mathcal{H}$ can be expressed as~\footnote{It is not necessarily true that for any $g \in L_2(d\tau)$, there exists a corresponding $f \in \mathcal{H}$.}
\begin{IEEEeqnarray}{rCl}
f(x) = \int_{\mathcal{V}}g(v)\psi(v,x)p(v)dv \qquad (\forall x \in \mathcal{X}) \label{fun_appx}
\end{IEEEeqnarray}
where $g \in L_2(d\tau)$ is a real-valued function such that $\|g\|_{L_2(d\tau)}^2 < \infty$ and $\|f\|_{\mathcal{H}}= \min_{g} \|g\|_{L_2(d\tau)}$, with the minimum taken over all possible decompositions of $f$. Thus, one can take an independent sample $\{v_i\}_{i=1}^s \sim p(v)$ (we refer to this sampling scheme as \emph{plain RFF}) and approximate a function $f \in \mathcal{H}$ by an element from $\tilde{\mathcal{H}}$ as
\begin{align*}
   \tilde{f}(\cdot) = \sum_{i=1}^s\alpha_i \psi(v_i,\cdot) = \phi_{s}(\cdot)^{\top}\alpha \quad \text{with} \quad \alpha \in \mathbb{R}^s\ .
\end{align*}

As the latter approximation is simply a Monte Carlo estimate, one could also select an importance weighted probability density function $q(\cdot)$ and sample features $\{v_i\}_{i=1}^s$ from $q$ (we refer to this sampling scheme as \emph{weighted RFF}).
The function $f$ can then be approximated by
\begin{align*}
  \tilde{f}_{q}(\cdot) = \sum_{i=1}^s\alpha_i \psi_q(v_i,\cdot) = \phi_{q,s}(\cdot)^{\top}\alpha \ ,
\end{align*}
with $\psi_q(v_i,\cdot)= \sqrt{p(v_i)/q(v_i)}\psi(v_i,\cdot)$ and $\phi_{q,s}(\cdot) = (1/\sqrt{s})[\psi_q(v_1,\cdot),\cdots,\psi_q(v_s,\cdot)]^{\top}$.

For both plain RFF and weighted RFF, the goal is to find $\tilde{f}$ with minimal norm such that the computation error between $\tilde{f}$ and $f$ is minimized. Similar to \citet{bach2017equivalence}, the RFFs sampling can be formulated as the following optimization problem
\begin{IEEEeqnarray}{rCl}
\|\tilde{f} - f\|_{L_2(P_x)}^2 + \lambda \|\tilde{f}\|_{\tilde{\mathcal{H}}}^2 \label{fun_opt} \ .
\end{IEEEeqnarray}
Note that since $\tilde{f} \notin \mathcal{H}$ in general, we use the $L_2(P_x)$ norm to measure the computation error.

\subsection{Integral Operator \& Leverage Score Sampling}
Kernel methods and RFFs are often studied through the integral operator $L:L_2(P_x) \rightarrow L_2(P_x)$, which we define below
\begin{IEEEeqnarray}{rCl}
(Lf)(\cdot) = \int_{\mathcal{X}}k(x,\cdot)f(x)dP_x(x)\ . \nonumber
\end{IEEEeqnarray}
Given the kernel decomposition as Eq.~(\ref{krl_dec}), the integral operator can be expressed as an expectation \citep{bach2017equivalence}
\begin{IEEEeqnarray}{rCl}
Lf &=& \int_{\mathcal{X}}k(x,\cdot)f(x)dP_x(x)\ , \nonumber\\
&=& \left(\int_{\mathcal{V}}\psi(v,\cdot)\otimes\psi(v,\cdot)p(v)dv\right)f \label{eq:integral_operator} \ ,
\end{IEEEeqnarray}
where $f\otimes g$ is the $L_2(P_x)$ outer product operator such that $\left(f\otimes g\right)h = \langle g, h \rangle_{L_2(P_x)}f$. Finally, if $k$ and $\psi$ are both bounded and continuous, then $L$ is positive definite, self-adjoint and trace-class. In particular, if $\|\psi\| \leq \kappa$, we have $\|L\| \leq \kappa^2$.

Similarly, for kernel $\tilde{k}$, we define the integral operator $L_s:L_2(P_x) \rightarrow L_2(P_x)$:
\begin{IEEEeqnarray}{rCl}
L_sf &=& \int_{\mathcal{X}}\tilde{k}(x,\cdot)f(x)dP_x(x)\ , \nonumber\\
&=& \int_{\mathcal{X}}\frac{1}{s}\sum_{i=1}^s\psi(v_i,\cdot)\psi(v_i,x)f(x)dP_x \ ,\nonumber \\
&=& \left(\frac{1}{s}\sum_{i=1}^s\psi(v_i,\cdot)\otimes\psi(v_i,\cdot)\right)f \label{eq:rff_integral_operator} \ .
\end{IEEEeqnarray}
Hence, $L_s$ can be seen as an empirical estimator of $L$.

The study of the integral operator is important because it provides information on how to select the optimal sampling distribution $q(v)$. A large body of literature shows that finding an optimal sampling distribution $q(v)$ often significantly reduces the number of features required \citep{bach2017equivalence,alaoui2015fast,avron2017random,rudi2017generalization}. The reason is that random features sampled according to $p(v)$ often focus on approximating the leading eigenvalues of the integral operator $L$. In contrast, a reweighted sampling distribution $q(v)$ allows the random features to span the whole eigenspectrum of $L$.

In light of this, a leverage score based weighted distribution function is first introduced in \citet{alaoui2015fast} in the context of the Nystr{\"o}m approximation~\citep{Nystrom30,Smola00,Williams01}. Utilizing the importance reweighted nature,~\citet{alaoui2015fast} establish a sharp convergence rate of the low-rank estimator based on the Nystr\"om method.

The success of the leverage score distribution further motivates the pursuit of a similar notion for RFFs. In particular,~\citet{bach2017equivalence} first proposes the leverage score sampling based on a leverage score function defined below
\begin{IEEEeqnarray}{rCl}
\tau_{\lambda}(v) = p(v)\langle \psi(v,\cdot), (L+\lambda I)^{-1}\psi(v,\cdot)\rangle_{L_2(P_x)} \ . \label{eq:lev_fun}
\end{IEEEeqnarray}
From our assumption, it follows that there exists a constant $\kappa$ such that $| \psi(v,x) | \leq \kappa$ (for all $v$ and $x$). We now have
\begin{align*}
\begin{aligned}
\tau_{\lambda}(v)\leq p(v)\frac{\kappa^2}{\lambda} \ .
\end{aligned}
\end{align*}
An important property of function $\tau_{\lambda}(v)$ is its relation to the effective number of parameters:
\begin{align*}
\begin{aligned}
\int_{\mathcal{V}}\tau_{\lambda}(v)dv = \text{Tr}\big[L(L+\lambda I)^{-1} \big]:= d(\lambda) \ ,
\end{aligned}
\end{align*}
where $d(\lambda)$ implicitly determines the number of parameters in a supervised learning problem and is thus called the \emph{number of effective degrees of freedom}~\citep{bach2013sharp,hastie2017generalized}.

We can now sample features according to $q^*(v) = \tau_{\lambda}(v)/d(\lambda)$, since $q^*(v)$ is a probability density function.~\citet{bach2017equivalence} studies the property of $q^*(v)$ and demonstrates that sampling according to $q^*(v)$ requires fewer Fourier features compared to the standard spectral measure sampling. From now on, we refer to $q^*(v)$ as the \emph{ridge leverage score distribution} and refer to this sampling strategy as \emph{leverage weighted RFF}.

\section{Main Results}
In this section, we provide our theoretical analysis on the trade-off between the number of random features and the statistical prediction accuracy. We first discuss the worst case scenario where the estimator achieves the $O(1/\sqrt{n})$ learning rate, followed by demonstrating the trade-off in the fast convergence rate setting.

\subsection{$O(1/\sqrt{n})$ Learning Rate}
We study the scenario where the RFFs estimator obtains the minimax learning rate $O(1/\sqrt{n})$. As discussed before, kernel supervised learning can be formulated as Eq.~(\ref{eq:krl_opm}). Since we are investigating the classification setting, we mainly consider the loss function $l$ to be uniformly Lipschitz continuous functions such as support vector machine and logistic regression. A fatal problem for kernel supervised learning is the computational cost since kernel learning problem such as Eq.~(\ref{eq:krl_opm}) often requires the store of the kernel Gram matrix $\mathbf{K}$ or even the inversion of $\mathbf{K}$, which are $O(n^2)$ and $O(n^3)$ computations respectively.

In order to overcome the computation issue, the RFFs provide an efficient way to approximate the kernel function. Specifically, we sample $v_1,\dots,v_s$ according to some importance sampling distribution $q(v)$ to form the random feature vector $\phi_s(\cdot)$. For a given data $(x,y)$, we then approximate the label $y$ with the random feature hypothesis $\tilde{f}_q(x) = \phi_{q,s}(x)^{\top}\beta$. The RFFs learning can be cast as the following optimization problem
\begin{IEEEeqnarray}{rCl}
	\tilde{f}^{\lambda}  &\coloneqq& \argmin_{\tilde{f}_q \in \tilde{\mathcal{H}}}\  \frac{1}{n}\sum_{i=1}^n l(y_i,\tilde{f}_q(x_i))+ \lambda \|\tilde{f}\|_{\tilde{\mathcal{H}}}^2\ . \nonumber 
\end{IEEEeqnarray}
According to \citet{bach2017equivalence} and \citet{li2019towards}, we have $\|\tilde{f}\|_{\tilde{\mathcal{H}}}^2 \leq \|\beta\|_2^2$, as a result, the above optimization can be reformulated as 
\begin{IEEEeqnarray}{rCl}
	\tilde{\beta}^{\lambda}  &\coloneqq& \argmin_{\beta \in \mathbb{R}^s}\  \frac{1}{n}\sum_{i=1}^n l(y_i,\phi_{q,s}^{\top}\beta)+ \lambda \|\beta\|_2^2\ . \label{eq:rff_cla_opt}
\end{IEEEeqnarray}
The RFFs hypothesis with loss function $l$ can be represented as $\tilde{f}^{\lambda} = \phi_{q,s}^{\top}\tilde{\beta}^{\lambda}$. Through the RFFs approximation, we now only need to store the feature matrix $\Phi_q = [\phi_{q,s}(x_1),\dots,\phi_{q,s}(x_n)]^{\top} \in \mathbb{R}^{n\times s}$. The inversion of $\mathbf{K}$ can be approximated as inverting $\Phi_q^{\top}\Phi_q \in \mathbb{R}^{s\times s}$. Hence the computation cost is now $O(ns)$ and $O(ns^2+s^3)$ respectively. We can see that if $s \ll n$, RFFs method enjoys a huge computational savings. However, a key question is how the choice of $s$ affects the prediction accuracy of $\tilde{f}^{\lambda}$.

In this section, we try to address the above issue. We first list our assumptions below
\begin{itemize}[leftmargin = 6.9mm]
    \item[A.$1$] We assume that the kernel has integral expansion as Eq.~(\ref{krl_dec}) such that $\psi(v,x)$ is continuous in both $v$ and $x$ and $|\psi(v,x)| \leq \kappa$ for all $x \in \mathcal{X}$ and $v \in \mathcal{V}$;
    
    \item[A.$2$] Assume that the loss function $l$ in Eq.~(\ref{eq:rff_cla_opt}) is uniformly Lipschitz continuous with constant $M$, i.e., \[\left|l(y,x_1)-l(y,x_2)\right| \leq M\|x_1-x_2\|_2.\]
    
    \item[A.$3$] Recall $f_{\mathcal{H}} = \argmin_{f \in \mathcal{H}}\ \mathbb{E}[l_f]$, we assume that \[f_{\mathcal{H}} = L^rg, ~~\textnormal{for~some~} r \in [1/2,1]~\&~ g \in L_2(P_x);\]
    
\end{itemize}
Assumptions A.$1$ and A.$2$ are standard assumptions made in classification problems. A.$3$ is a regularity condition that is commonly used in approximation theory \citep{smale2003estimating}. It describes the decay rate of the coefficients of $f_{\mathcal{H}}$ along the basis given by the integral operator $L$, which further allows controlling the bias of the estimator. While being overlooked in the classification setting, A.$3$ is a key property used in RFFs regression to obtain a better computation and accuracy trade-off. Utlizing A.$3$ enables us to prove the following refined analysis.

\begin{restatable}{theorem}{thmworLipsch}\label{theo:minx}
Assume A.$1$, A.$2$ and A.$3$ hold. Suppose we have a measurable function $\tilde{\tau}: \mathcal{V} \rightarrow \mathbb{R}$ such that $\tilde{\tau}(v) \geq \tau_{\lambda}(v)$ almost surely. Denote $d_{\tilde{\tau}} = \int_{\mathcal{V}}\tilde{\tau}(v) dv$, and let $q(v) = \frac{\tilde{\tau}(v)}{d_{\tilde{\tau}}}$. We sample $v_1,\dots,v_s \sim q(v)$ and compute the hypothesis $\tilde{f}^{\lambda}$ by solving the optimization problem in Eq.~(\ref{eq:rff_cla_opt}). Let $\delta \in (0,1)$, if we have \vspace{-.5em} \[s \geq 12 d_{\tilde{\tau}}\log\frac{d(\lambda)}{\delta},\] with probability over $1-\delta$, 
\begin{IEEEeqnarray}{rCl}
\mathbb{E}( l_{\tilde{f}^{\lambda}}) - \mathbb{E} \left( f_{\mathcal{H}}\right) \leq 2MR\lambda^{r} + O\left(1/\sqrt{n}\right) \ . \label{eq:risk_wor}
\end{IEEEeqnarray}
\end{restatable}

Theorem \ref{theo:minx} expresses the trade-off between the computational cost and statistical efficiency through the regularization parameter $\lambda$, the effective dimension of the problem $d(\lambda)$, and the normalization constant $d_{\tilde{\tau}}$ of the sampling distribution. The regularization parameter $\lambda$ is used as a key quantity in the analysis of supervised learning setting \citep{caponnetto2007optimal,rudi2017generalization,li2019towards}. In particular, if we set $\lambda \propto 1/n^{2r}$, we observe that the estimator $\tilde{f}^{\lambda}$ attains the $O(1/\sqrt{n})$ learning rate \citep{bach2017equivalence}. As a consequence of Theorem~\ref{theo:minx}, we have the following bounds on the number of required features for the two strategies: \emph{plain} RFF (Corollary \ref{cla_risk_cor_p}) and \emph{leverage weighted} RFF (Corollary \ref{cla_risk_cor_q}).

\begin{cor}\label{cla_risk_cor_p}
If the probability density function from Theorem~\ref{theo:minx} is the spectral measure $p(v)$, then the upper bound on the learning risk from Eq.~(\ref{eq:risk_wor}) holds for all $s \geq 5\kappa^2/\lambda\log\frac{16d(\lambda)}{\delta}$.
\end{cor}
\begin{proof}
We set $\tilde{l}(v) = p(v)\kappa^2/\lambda$ and obtain $d_{\tilde{l}} = \int_{\mathcal{V}}p(v)\kappa^2/\lambda dv = \kappa^2/\lambda$.
\end{proof}

Theorem \ref{theo:minx} and Corollary \ref{cla_risk_cor_p} have several implications on the choice of $\lambda$ and $s$ in the classification setting with plain RFF. In particular, the usual generalization bound for kernel estimator $\hat{f}^{\lambda}$ (i.e., minimizer of Eq.~(\ref{eq:krl_opm})) is $O(1/\sqrt{n})$~\citep[see e.g.,][]{rahimi2009weighted,shalev2014understanding,bach2017equivalence}. As such, if we set $\lambda = O(n^{1/2r})$, we can see that the RFFs estimator $\tilde{f}^{\lambda}$ incurs no loss of prediction accuracy while offering computational gains.

Specifically, in the benign case where $r = 1$, $O(\sqrt{n}\log n)$\footnote{We use the fact that $d(\lambda) \ll n$} features is able to achieve the $O(1/\sqrt{n})$ learning rate. Comparing with the existing analysis where $O(n\log n)$ features are required \citep{rahimi2009weighted,li2019towards}, our result is a significant improvement. We also achieve remarkable computational savings: from roughly $O(n^3)$ and $O(n^2)$ in time and space for original kernel methods to $O(n^2)$ and $O(n\sqrt{n})$ for the RFFs approximation. Moreover, when $r > 1/2$, we also obtain computational gain as the number of features required now is $\Omega(n^{1/2r})$ with $2r >1$. In the worst scenario where $r = 1/2$ (equivalent to assuming $f_{\mathcal{H}}$ exists), we recover the results from existing analysis \citep{rahimi2007random,li2019towards}.

To our knowledge, this is the first result showing that for a large class of classification problems ($r > 1/2$), RFFs classification can dramatically reduce the computational cost while preserving the optimal generalization properties.

\begin{cor}\label{cla_risk_cor_q}
If the probability density function from Theorem~\ref{theo:minx} is the ridge leverage score distribution $q^*(v)$, the upper bound on the risk from Eq.~(\ref{eq:risk_wor}) holds for all $s \geq  5d(\lambda)\log\frac{16d(\lambda)}{\delta}$.
\end{cor}

\begin{proof}
For this corollary, we set $\tilde{\tau}(v) = \tau_{\lambda}(v)$ and deduce $d_{\tilde{l}} = \int_{\mathcal{V}}\tau_{\lambda}(v)dv = d(\lambda)$. 
\end{proof}

Corollary \ref{cla_risk_cor_q} details the number of features required in the leverage weighted RFF setting. Similar to the plain RFF setting, the RFFs estimator obtains $O(1/\sqrt{n})$ rate once we set $\lambda = O( n^{1/2r})$. However, the choice of $s$ now is determined by two factors: the regularity condition $r$ and the decay rate of the eigenspectrum of $L$.

We first consider the benign scenario where $r = 1$. Depending on the eigenspectrum decay rate, we have several different cases. Denote $\{\mu_1,\mu_2,\dots,\}$ to be the eigenvalue of $L$, in the best case where $L$ has finite rank, $d(\lambda)$ remains constant as $n$ grows. We therefore conclude that even $\Omega(1)$ features can guarantee the $O(1/\sqrt{n})$ learning rate. Next, if the eigenspectrum displays exponential decay, i.e.,~$\mu_i \propto C_0r^i$, we have $d(\lambda) \leq \log (C_0^0/\lambda)$. We can see that $s \geq \log n \log \log n$ is enough to achieve the $O(1/\sqrt{n})$ learning rate. As such, significant computational savings is obtained: from $O(n^3)$ and $O(n^2)$ to $O(n\log^4 n)$ and $O(n \log n)$ respectively. In the case of a slower decay with $\mu_i \propto C_0 i^{-2\gamma}$, we have $d(\lambda) \leq (R_0/\lambda)^{1/(2\gamma)}$ and $s\geq n^{1/4\gamma}\log n$. Hence, substantial computational savings can be achieved even in this case. Furthermore, in the worst case with $\mu_i$ close to $C_0 i^{-1}$, our bound implies that $s \geq n^{1/2}\log n$ features are sufficient.

The analysis for the worst case scenario where $r = 1/2$ is similar. The required numbers of features are $\Omega(1), \Omega(\log n \log \log n), \Omega(n^{1/2\gamma})$ and $\Omega(n\log n)$ for the cases where the eigenspectrum has finite rank, decays exponentially, proportional to $C_0 i^{-2\gamma}$ and close to $C_0i^{-1}$, respectively. Our results demonstrate that huge computational savings are possible as long as the eigenspectrum of $L$ displays fast decay (faster than $i^{-1}$).

\begin{table*}[t]
	\centering\fontsize{10}{20}\selectfont
	\begin{tabular}{l|c|c|c|c}
		
		\hline
		
		\hline
		
		\textsc{sampling scheme}&\textsc{spectrum} & \textsc{this work} &\textsc{li et al. (2021)}& \textsc{learning rate}\\\cline{1-5}
		
		\multirow{3}{7em}{\sc plain rff} & \textsc{finite rank} & $s \in \Omega (n^{1/2r})$& $s \in \Omega (n)$ & \multirow{3}{3em}{$O(1/\sqrt{n})$}\\\cline{2-4}
		
		&\textsc{exponential decay} &$ s \in \Omega(n^{1/2r} \cdot\log \log n)$ &$ s \in \Omega(n\cdot\log \log n)$ & 	\\\cline{2-4}
		
		&\textsc{polynomial decay} &$s \in \Omega(n^{1/2r}\cdot \log n)$ &$s \in \Omega(n\cdot \log n)$ & 	\\
		
		\cline{1-5}
		
		\cline{1-5}

		\multirow{4}{7em}{\textsc{weighted rff}}& \textsc{finite rank} & $s \in \Omega(1)$& $s \in \Omega(1)$ &\multirow{4}{3em}{$O(1/\sqrt{n})$} \\\cline{2-4}
		
		&\textsc{Exponential Decay} & $s \in \Omega (\log n \cdot \log \log n)$ &$s \in \Omega (\log n \cdot \log \log n)$  & 	\\\cline{2-4}
		
		&$\mu_i \propto i^{-2\gamma}, \gamma \geq 1$ & $s \in \Omega (n^{1/4\gamma r} \cdot \log n)$ &$s \in \Omega (n^{1/2\gamma} \cdot \log n)$  &	\\\cline{2-4}
		
		&$\mu_i \propto i^{-1}$ & $s \in \Omega (n^{1/2r} \cdot \log n)$ &$s \in \Omega (n \cdot \log n)$  &	\\
		\cline{1-5}
		\hline
	\end{tabular}
	\caption{The comparison of our results to the sharpest learning rates from prior work \citep{li2021towards}, where $r \in [1/2,1]$  .}\label{tab:wor_lip_pla}
\end{table*}

\subsubsection{Comparison with Existing Sharpest Results}
Under $O(1/\sqrt{n})$ learning rate setting, \citet{rahimi2009weighted}, \citet{bach2017equivalence}, and \citet{li2021towards} analyze the trade-off between the number of features and the statistical prediction accuracy. Table~\ref{tab:wor_lip_pla} provides a detailed comparison between this work and that from \citet{li2021towards}. Our results show that we obtain at least the same rate as the previous best rate in the literature when $r = 1/2$, while significantly improving the trade-off under benign conditions ($r > 1/2$). 

The first block of rows in Table~\ref{tab:wor_lip_pla} illustrates the difference between our work and that from \citet{li2021towards} when plain sampling is used. A key feature in our results is that when $r > 1/2$, our results state that the RFFs approximation provides computational gain without trading off for the prediction accuracy. In comparison, results from \citet{li2021towards} state that there is no computational gain ($s = \Omega(n)$) if we were to achieve the $O(1/\sqrt{n})$ learning rate. In addition, we recover the results from \citet{li2021towards} when $r = 1/2$.

We observe a similar pattern when leverage weighted RFF is used. In particular, our results match those from \citet{li2021towards}, when the eigenspectrum has finite rank or displays exponential decay. However, as soon as the eigenspectrum has polynomial decay, our result is sharper. Specifically, when the eigenvalue decays polynomially with $\mu_i \propto i^{-2\gamma}$ and $r = 1$, our results show that $\Omega(n^{1/4\gamma \log n})$ features are enough to achieve $O(1/\sqrt{n})$ learning rate, comparing with $\Omega(n^{1/2\gamma})$ features required from \citet{li2021towards}. When the eigenvalue decays close to $i^{-1}$, our results require $\Omega(n^{1/2}\log n)$ features while \citet{li2021towards} require $\Omega(n \log n)$ features.

\subsection{Refined Learning Rate}
In the previous section, we study the trade-off between the number of features and the statistical prediction accuracy in the $O(1/\sqrt{n})$ minimax learning rate setting. In general, it is hard for classification problems to obtain learning rates sharper than $O\left(1/\sqrt{n} \right)$. However, under some benign conditions, it is possible to obtain $O(1/n)$ convergence rate as demonstrated by \citet{bartlett2006convexity} and \citet{steinwart2008support}. As such, with the help of the following assumption, we derive a sharp learning rate for RFFs classification problems in this section.
\begin{enumerate}[leftmargin = 6.9mm]
    \item[A.$4$] Recall that $f_{\mathcal{H}}$ is the optimal estimator in A.$3$. We assume that there exists a constant $G$ such that for all $f \in \mathcal{H}$ \[\mathbb{E}[(f - f_{\mathcal{H}})^2] \leq G\mathbb{E}[l_f - l_{f_{\mathcal{H}}}] \ .\]  
\end{enumerate}
Assumption A.$4$ is a widely used condition for classification problems to obtain faster learning rates. It typically requires that the loss function $l$ is uniformly convex and the function space $\mathcal{H}$ is convex and uniformly bounded. It can be shown that many loss functions satisfy this assumption, including squared loss \citep{bartlett2005local} and hinge loss \citep[Chapter 8.5]{steinwart2008support}. Additional examples of these loss functions are discussed in \citet{bartlett2006convexity} and \citet{mendelson2002improving}. In addition, since $l$ is Lipschitz continuous, we can rewrite A.$4$ as \[\mathbb{E}[(l_f-l_{f_{\mathcal{H}}})^2] \leq L^2\mathbb{E}[(f - f_{\mathcal{H}})^2] \leq G L^2\mathbb{E}[l_f - l_{f_{\mathcal{H}}}] \ .\] This is the variance condition described in \citet[Chapter 7.3]{steinwart2008support}, which is also linked to the Massart's low noise condition or more generally to the Tsybakov condition \citep{sun2018but}. Intuitively speaking, the condition requires that the bayes classifier $P(Y=1 \mid X=x)$ is not close to $1/2$~\citep[see e.g.,][for more details]{tsybakov2004optimal,koltchinskii2011oracle}.

\begin{restatable}{theorem}{thmrefLipsch}\label{theo:fast_rate}
Assume A.$1$-A.$4$ hold. In addition, we assume the condition for $\tilde{\tau}$, $d_{\tilde{\tau}}$ and $q(v)$ hold as that in Theorem~\ref{theo:minx}. Let $\{\tilde{\mu}_1,\tilde{\mu}_2,\dots\}$ be the eigenvalues of the normalized Gram-matrix $(1/n) \tilde{\mathbf{K}}$, $c_1, c_2, c_3$ be some universal constant, and $\delta \in (0,1)$, if we have \vspace{-.5em} \[s \geq 12 d_{\tilde{\tau}}\log\frac{d(\lambda)}{\delta},\] with probability over $1-\delta$, 
\vspace{-.5em}
\begin{IEEEeqnarray}{rCl}
\mathbb{E}( l_{\tilde{f}^{\lambda}}) - \mathbb{E} \left( f_{\mathcal{H}}\right) \leq2MR\lambda^{r} + c_1\hat{r}^* + \frac{c_2}{n}\log \frac{1}{\delta}  \ , \nonumber
\end{IEEEeqnarray}
where \vspace{-.5em}\[\hat{r}^* \leq c_3 \min _{0 \leq h \leq n}\left(\frac{h}{n}+\sqrt{\frac{1}{n} \sum_{i>h} \tilde{\mu}_{i}}\right).\]
\end{restatable}

Theorem \ref{theo:fast_rate} covers a wide range of cases and can provide sharp risk convergence rates. In particular, $\hat{r}^*$ has an upper bound of $O(1/\sqrt{n})$ in all cases, which happens when $\hat{\mu}_i$ decays polynomially as $O(n^{-\gamma})$ with $\gamma > 1$ and we let $h = 0$. On the other hand, if $\hat{\mu}_i$ decays exponentially, then setting $h = \ceil{\log n}$ implies that $\hat{r}^* \leq O(\log n /n)$. In the best case, when $(1/n) \mathbf{K}$ has only finite rank $d_K$, then $\hat{r}^* \leq O(1/n)$ by letting $h = d_K+1$. These different upper bounds provide insights into various trade-offs between computational complexity and statistical efficiency. We now split the discussion into two cases: plain RFF and leverage weighted RFF.

Under plain RFF strategy, similar to Corollary~\ref{cla_risk_cor_p}, we have $d_{\tilde{\tau}} \leq \kappa^2/\lambda$. If the eigenvalues decay polynomially, i.e., $\mu_i \propto i^{-\gamma}$ with $\gamma > 1$, then the learning rate is upper bounded by $O(1/\sqrt{n})$. In this case, we need $s =\Omega(n^{1/2r}\log n)$. On the other hand, if $\mu_i$ decays exponentially, $\Omega(n^{1/r})$ features are able to guarantee $O(\log n / n)$ learning rate. Finally, if the eigenspectrum has finite rank, $\Omega(n^{1/r})$ features yield $O(1/n)$ fast learning rate. Under the leverage weighted RFF, the required numbers of features and the corresponding learning rates for the three above cases are: \emph{i}) $s =\Omega(n^{1/4\gamma r} \log n)$ and $O(1/\sqrt{n})$ (polynomial decay $\mu_i \propto i^{-\gamma}$ with $\gamma > 1$); \emph{ii}) $s=\Omega(\log n \log \log n)$ and $O(\log n/n)$ (exponential decay); and \emph{iii}) $s = \Omega(1)$ and $O(1/n)$ (finite rank). 

\begin{table*}[t]
	\centering\fontsize{10}{20}\selectfont
	\begin{tabular}{l|c|c|c|c}
		
		\hline
		
		\hline
		
		\textsc{sampling scheme}&\textsc{results}&\textsc{spectrum} & \textsc{number of features} & \textsc{learning rate}\\\cline{1-5}
		
		\multirow{6}{7em}{\sc plain rff}&\multirow{3}{6em}{\textsc{this work}} & 
		\textsc{finite rank} & $s \in \Omega (n^{1/r})$& $O(1/n)$ \\\cline{3-5}
		
		&&\textsc{exponential decay}  &$ s \in \Omega(n^{1/r})$ &$O(\log n / n)$ 	\\\cline{3-5}
		
	    &&$\mu_i \propto i^{-\gamma}$  &$s \in\Omega(n^{1/2r}\cdot \log n)$ & $O(1/\sqrt{n})$	\\\cline{3-5}
		
		\cline{2-5}

		\cline{2-5}
		&\multirow{3}{6em}{\textsc{li et al. (2021)}} & \textsc{finite rank} &$s \in\Omega(n^2)$ & $O(1/n)$\\\cline{3-5}
		
		&&\textsc{exponential decay}  &$s \in\Omega(n^2)$ & $O(\log n /n)$	\\\cline{3-5}
		
		&&$\mu_i \propto i^{-\gamma}$  & $s \in\Omega(n\cdot \log n)$& $O(1/\sqrt{n})$ 	\\\cline{3-5}
		
		\cline{1-5}
		\cline{1-5}
		\cline{1-5}

		\multirow{6}{7em}{\textsc{weighted rff}}&\multirow{3}{6em}{\textsc{this work}}& \textsc{finite rank} & $s \in \Omega(1)$& $O(1/n)$\\\cline{3-5}
		
		&&\textsc{exponential decay} & $s \in \Omega (\log n \cdot \log \log n)$ & $O(\log n / n)$	\\\cline{3-5}
		
		&&$\mu_i \propto i^{-\gamma}$ & $s \in \Omega (n^{1/4\gamma r} \cdot \log n)$ & $O(1/\sqrt{n})$	\\\cline{3-5}
		
		\cline{2-5}
		
		\cline{2-5}
		&\multirow{3}{6em}{\textsc{sun (2018)}} & \textsc{finite rank} & $s \in \Omega (1)$& $O(1/n)$\\\cline{3-5}
		
		&&\textsc{exponential decay} &$s \in \Omega (\log^d n \cdot \log \log^d n)$  & $O(\log^{d+2} n / n)$	\\\cline{3-5}
		
		&&$\mu_i \propto i^{-\gamma}$ &$s \in \Omega (n^{\frac{2}{2+\gamma}} \cdot \log n)$  & $O(1/n^{\frac{\gamma}{2+\gamma}})$ 	\\\cline{3-5}
		
		\hline
		
		\hline
		\end{tabular}
	\caption{The comparison of our results to the sharpest results in the literature under fast learning rate setting, where $r \in [1/2,1]$.}\label{tab:ref_lip_wei}
\end{table*}
\subsubsection{Comparison with Existing Sharpest Results}
For fast learning rate scenario, \citet{li2021towards} and \citet{sun2018but} both study the trade-off between the number of features and the prediction accuracy and obtain a similar bound on the number of features required. Table~\ref{tab:ref_lip_wei} compare our results with that in \citet{li2021towards} under plain RFF sampling, and that in \citet{sun2018but} under leverage weighted RFF sampling. Similar to the analysis in the $O(1/\sqrt{n})$ scenario, our results strictly dominate previous optimal results in \citet{li2021towards} when $r > 1/2$ and match the obtained bound in \citet{li2021towards} when $r = 1/2$, because of Assumption A.$3$.

Under weighted RFF, our results match that in \citet{sun2018but} when the eigenspectrum has finite rank. However, when the eigenspectrum displays exponential decay, results from \citet{sun2018but} suffer from the curse of dimension, since both the number of features required and the learning rate obtained depend on the data dimension $d$. In contrast, our analysis does not have this dependency. When the eigenspectrum exhibits a polynomial decay, our results achieve the $O(1/\sqrt{n})$ learning rate while \citet{sun2018but} obtain a more flexible rate that depends on $\gamma$. If $\gamma \leq 2$, our results have a better trade-off as both the number of features and the learning rate are sharper than those from~\citet{sun2018but}. For example, if $\gamma = 2$, analysis in \citet{sun2018but} shows that $\Omega(n^{1/2}\log n)$ features can obtain $O(n^{-1/2})$ learning rate, whereas our results state that $\Omega(n^{1/8r}\log n )$ features yield the same learning rate. When $\gamma > 2$, the learning rate obtained by \citet{sun2018but} is faster than ours at the cost of increasing the number of features, i.e., $ \Omega (n^{\frac{2}{2+t}} \cdot \log n)$ versus $ \Omega(n^{\frac{1}{4\gamma r}} \cdot \log n)$. In particular, setting $\gamma=4$, we can see that \citet{sun2018but} obtain a fast $O(n^{2/3})$ learning rate, but at the cost of requiring $ \Omega(n^{1/3}\cdot \log n)$ random features. For the same setting, on the other hand, we obtain the minimax optimal $O(1/\sqrt{n})$ learning rate with only $\Omega(n^{1/16r}\cdot \log n)$ random features.

\section{Conclusion}
In this paper, we thoroughly study the generalization properties of RFFs classification with Lipschitz continuous loss such as support vector machine and logistic regression. Our main results for the first time demonstrate that RFFs classification can indeed provide computational gains without hurting the prediction accuracy when plain RFF is used. This is in contrast with all previous results that suggest that computational savings come at the expense of prediction accuracy. Furthermore, our analysis shows that a fast $O(1/n)$ learning rate is possible at the cost of increasing the number of features unless the leverage weighted RFF is used. However, a limitation in our work is that in the worst case where $r = 1/2$, we can see that the current analysis on RFFs classification cannot guarantee a computational gain. Therefore, how to obtain a sharper result in the worst case is an interesting future direction. In addition, how to efficiently approximate the leverage score is also an important future direction since it can often leads to significant reduction in the required number of feature as well as computational cost.


\bibliographystyle{plainnat}  
\bibliography{ref}

\newpage
\appendix
\section{Notation \& Definition}
In the rest of the appendix, we will denote $P(x,y)$ to be the joint probability density function on $\mathcal{X} \times \mathcal{Y}$, where $\mathcal{X}$ is the instance space and $\mathcal{Y}$ is the label space. We will use $\|A\|$ to denote the vector norm if $A$ is a vector and operator norm if $A$ is a matrix or an operator. We will use $L_2(P_x)$ to denote the space of square-integrable functions with respect to $P_x$, the mariginal distribution of $P(x,y)$ with $\|\cdot\|_{L_2(P_x)}$ being the norm function and $\langle \cdot, \cdot \rangle_{L_2(P_x)}$ being the inner product. In addition, we use $\otimes$ to denote the $L_2(P_x)$ outer product. We will also use Tr$(A)$ to denote the trace of matrix $A$ or operator $A$. 

Our analysis heavily depends on the notion of Rademacher complexity which we define below.
\begin{defn}[Rademacher Complexity]\label{def:rade}
Suppose $\{x_i\}_{i=1}^n$ are i.i.d samples from $P_x$. Let $\mathcal{H}$ be a class of functions mapping $\mathcal{X}$ to $\mathbb{R}$. Then, the random variable known as the \textit{empirical Rademacher complexity} is defined as 
\begin{IEEEeqnarray}{rCl}
\hat{R}_n(\mathcal{H}) = \mathbb{E}_{\sigma}\Bigg[\sup_{f\in\mathcal{H}}\Bigg|\frac{2}{n}\sum_{i=1}^n\sigma_if(x_i)\Bigg|\mid x_1,\cdots,x_n   \Bigg] \ ,\nonumber
\end{IEEEeqnarray}
where $\sigma_1,\cdots,\sigma_n$ are independent Rademacher random variables. The corresponding \textit{Rademacher complexity} is then defined as $$R_n(\mathcal{H}) = \mathbb{E}\Big[\hat{R}_n(\mathcal{H})\Big] \ ,$$  
where the expectation is taken with respect to $P_x$.
\end{defn}

Finally, for random Fourier feature vector $\phi_{q,s}(\cdot)$, we define the operator $\Psi_s : \mathbb{R}^s \rightarrow L_2(P_x)$ as
\begin{IEEEeqnarray}{rCl}
\Psi_s\beta = \phi_{q,s}^{\top}\beta \ .\label{eq:psi_operator}
\end{IEEEeqnarray}

\section{Random Fourier Features Approximation of $f_{\mathcal{H}}$}
In this section, we consider using RFFs to approximate the optimal function $f_{\mathcal{H}}$ in $\mathcal{H}$. In other words, we first sample $v_1,\dots,v_s$ from some importance sampling distribution $q(v)$ and approximate $f_{\mathcal{H}}$ with $\phi_{q,s}(\cdot)^{\top}\beta$ for some $\beta \in \mathbb{R}^s$. To find the best $\beta$, we form the following optimization problem:
\begin{IEEEeqnarray}{rCl}
\tilde{\beta}_{\mathcal{H}}:= \argmin_{\beta \in \mathbb{R}^s} \|\phi_{q,s}^{\top}\beta - f_{\mathcal{H}}\|_{L_2(P_x)}^2 + \lambda \|\beta\|_{2}^2 \label{fun_H_opt}.
\end{IEEEeqnarray}
We denote the $\tilde{f}_{\mathcal{H}}$ as the RFFs approximation of $f_{\mathcal{H}}$. Our next theorem indicates the approximation property of $\tilde{f}_{\mathcal{H}}$.

\begin{theorem}\label{theo:f_H_approx_error}
Assume A.$1$ and A.$3$ hold. Suppose we have a measurable function $\tilde{\tau}: \mathcal{V} \rightarrow \mathbb{R}$ such that $\tilde{\tau}(v) \geq \tau_{\lambda}(v)$ almost surely, denote $d_{\tilde{\tau}} = \int_{\mathcal{V}}\tilde{\tau}(v) dv$, and let $q(v) = \frac{\tilde{\tau}(v)}{d_{\tilde{\tau}}}$. Let $\delta \in (0,1)$, if we sample $v_1,\dots,v_s \sim q(v)$ such that \[s \geq 12 d_{\tilde{\tau}}\log\frac{d(\lambda)}{\delta},\] with probability over $1-\delta$ \[\left\| \tilde{f}_{\mathcal{H}}- f_{\mathcal{H}}\right\|_{L_2(P_x)} \leq 2R\lambda^{r}.\]
\end{theorem}

\begin{proof}
We approximate $f_{\mathcal{H}}$ with $\tilde{f}_q(\cdot) = \phi_{q,s}^{\top}\beta$ and formulate the optimization problem in Eq.~(\ref{fun_H_opt}). According to~Eq.~(\ref{eq:psi_operator}), we can further rewrite Eq.~(\ref{fun_H_opt}) as 
\begin{IEEEeqnarray}{rCl}
\left\|\Psi_s\beta - f_{\mathcal{H}}\right\|_{L_2(P_x)}^2 + \lambda \|\beta\|_2^2 \nonumber \ .
\end{IEEEeqnarray}
The solution of the above optimization problem can be computed as \[\tilde{\beta}_{\mathcal{H}} = \Psi_s^*\left(\Psi_s\Psi_s^* + \lambda I \right)^{-1}f_{\mathcal{H}} \,\]
where $\Psi^*$ denotes the adjoint operator of $\Psi$. Hence, we have 
\begin{IEEEeqnarray*}{rCl}
\tilde{f}_{\mathcal{H}}- f_{\mathcal{H}} &=& \Psi_s\tilde{\beta}_{\mathcal{H}} - f_{\mathcal{H}}  \ , \\
&=&   \Psi_s\Psi_s^*\left(\Psi_s\Psi_s^* + \lambda I \right)^{-1}f_{\mathcal{H}} -f_{\mathcal{H}} \ , \\
&=&  L_s\left(L_s + \lambda I \right)^{-1}f_{\mathcal{H}} -f_{\mathcal{H}} \ , \\
&=&  \lambda \left(L_s + \lambda I \right)^{-1}f_{\mathcal{H}} \ ,\\
& =& \lambda \left(L + \lambda I \right)^{-\frac{1}{2}}\left(I + \left(L + \lambda I \right)^{-\frac{1}{2}}(L_s - L) \left(L + \lambda I \right)^{-\frac{1}{2}}\right)^{-1} \left(L + \lambda I \right)^{-\frac{1}{2}} f_{\mathcal{H}} \ ,\\
&=& \lambda \left(L + \lambda I \right)^{-\frac{1}{2}}\left(I + \left(L + \lambda I \right)^{-\frac{1}{2}}(L_s - L) \left(L + \lambda I \right)^{-\frac{1}{2}}\right)^{-1} \left(L + \lambda I \right)^{-\frac{1}{2}} L^rg \ ,
\end{IEEEeqnarray*}
where for last step we used assumption A.$3$.

In the mean time, we have 
\begin{IEEEeqnarray*}{rCl}
\left\|\tilde{f}_{\mathcal{H}}\right\| &\leq& \left\| \tilde{\beta}_{\mathcal{H}}\right\|_2^2 = \langle f_{\mathcal{H}},\left(L_s + \lambda I \right)^{-1}L_s\left(L_s + \lambda I \right)^{-1}f_{\mathcal{H}}\rangle_{L_2(P_x)} \ ,\\
&\leq & \langle f_{\mathcal{H}},\left(L_s + \lambda I \right)^{-1} f_{\mathcal{H}}\rangle_{L_2(P_x)}\ , \\
& =& \langle f_{\mathcal{H}},\left(L + \lambda I \right)^{-\frac{1}{2}}\left(I + \left(L + \lambda I \right)^{-\frac{1}{2}}(L_s - L) \left(L + \lambda I \right)^{-\frac{1}{2}}\right)^{-1} \left(L + \lambda I \right)^{-\frac{1}{2}} f_{\mathcal{H}}\rangle_{L_2(P_x)}
\end{IEEEeqnarray*}

We now utilize Lemma~\ref{lma:inte_apprx} to lower bound $\left(L + \lambda I \right)^{-\frac{1}{2}}(L_s - L) \left(L + \lambda I \right)^{-\frac{1}{2}}$. Specifically, let $\epsilon = 1/2$ in Lemma~\ref{lma:inte_apprx}, we obtain that when \[s \geq 12 d_{\tilde{\tau}}\log\frac{d(\lambda)}{\delta},\] we have with probability over $1-\delta$ such that \[\left\|\left(L + \lambda I \right)^{-\frac{1}{2}}(L_s - L) \left(L + \lambda I \right)^{-\frac{1}{2}}\right\| \preceq \frac{1}{2}I.\]

Therefore, 
\begin{IEEEeqnarray*}{rCl}
\left\|\tilde{f}_{\mathcal{H}}- f_{\mathcal{H}}\right\|_{L_2(P_x)}&\leq & (1-\frac{1}{2})^{-1} \left\|\lambda \left(L + \lambda I \right)^{-1}L^rg\right\|_{L_2(P_x)} \ ,\\
&=& 2\left\|\lambda \left(L + \lambda I \right)^{-(1-r)} \left(L + \lambda I \right)^{-r}L^rg\right\|_{L_2(P_x)} \ ,\\
&\leq& 2 \lambda^{r}\|g\|_{L_2(P_x)} \ ,\\
&=& 2R\lambda^r\ .
\end{IEEEeqnarray*}

Finally, recall $r \in [1/2,1]$, we have \[\|\tilde{\beta}_{\mathcal{H}}\|_2^2 \leq 2 \langle f_{\mathcal{H}},\left(L + \lambda I \right)^{-1} f_{\mathcal{H}}\rangle_{L_2(P_x)}= 2 \langle g,\left(L + \lambda I \right)^{-1} L^{2r}g\rangle_{L_2(P_x)} \leq 2R^2.\]
\end{proof}

\section{Approximation Error of Integral Operator}
The following lemma characterizes the approximation error of $L$ using $L_s$.

\begin{lma}\label{lma:inte_apprx}
Let $L$ and $L_s$ be defined as in Eq.~(\ref{eq:integral_operator}) and Eq.(\ref{eq:rff_integral_operator}) respectively and recall we denote the leverage score as $\tau_{\lambda}(v)$. Suppose we have a measurable function $\tilde{\tau}: \mathcal{V} \rightarrow \mathbb{R}$ such that $\tilde{\tau}(v) \geq \tau_{\lambda}(v)$ almost surely, denote $d_{\tilde{\tau}} = \int_{\mathcal{V}}\tilde{\tau}(v) dv$, and let $q(v) = \frac{\tilde{\tau}(v)}{d_{\tilde{\tau}}}$. Assume A.$1$ holds and we draw random features $v_1\dots,v_s \sim q(v)$, for some constant $\epsilon > 0$ and $\delta \in (0,1)$, if the number of features satisfies \[s \geq \left(\frac{2}{\epsilon^2} + \frac{4}{3\epsilon} \right) d_{\tilde{\tau}}\log\frac{d(\lambda)}{\delta},\] then with probability over $1-\delta$, we have \[\left\|\left(L+ \lambda I\right)^{-\frac{1}{2}}(L_s - L)\left(L+ \lambda I\right)^{-\frac{1}{2}} \right\| \leq \epsilon.\]
\end{lma}
\begin{proof}
We employ the Bernstein inequality to prove the lemma. To this end, we first define \[\mathbf{R}_i = \frac{p(v_i)}{q(v_i)}\left(L+ \lambda I\right)^{-\frac{1}{2}}\psi(v_i,\cdot)\otimes\psi(v_i,\cdot) \left(L+ \lambda I\right)^{-\frac{1}{2}},\] then by Eq.~(\ref{eq:rff_integral_operator}) we have \[\left(L+ \lambda I\right)^{-\frac{1}{2}} L_s\left(L+ \lambda I\right)^{-\frac{1}{2}} = \frac{1}{s}\sum_{i=1}^s \mathbf{R}_i.\] We then immediately have $\mathbb{E}\left(\mathbf{R}_i\right) = \left(L+ \lambda I\right)^{-\frac{1}{2}} L\left(L+ \lambda I\right)^{-\frac{1}{2}}$.

Furthermore, we have
\begin{IEEEeqnarray*}{rCl}
\left\| \mathbf{R}_i\right\| &=& \frac{p(v_i)}{q(v_i)}\left\|\left(L+ \lambda I\right)^{-\frac{1}{2}}\psi(v_i,\cdot)\otimes\psi(v_i,\cdot) \left(L+ \lambda I\right)^{-\frac{1}{2}} \right\| \ ,\\
&=& \frac{p(v_i)}{q(v_i)} \langle \psi(v_i,\cdot), \left(L+ \lambda I\right)^{-1}\psi(v_i,\cdot) \rangle_{L_2(P_x)} \ , \\
&=& \frac{\tau_{\lambda}(v_i)}{q(v_i)} =  \frac{\tau_{\lambda}(v_i)}{\tilde{\tau}(v_i)}d_{\tilde{\tau}} \ ,\\
&\leq & d_{\tilde{\tau}}\ .
\end{IEEEeqnarray*}

In addition, $\mathbf{R}_i$ is a self-adjoint operator, we have 
\begin{IEEEeqnarray*}{rCl}
\mathbf{R}_i\mathbf{R}_i^* &=& \mathbf{R}_i^*\mathbf{R}_i \ ,\\
&=& \frac{p(v_i)^2}{q(v_i)^2}\left(L+ \lambda I\right)^{-\frac{1}{2}}\psi(v_i,\cdot)\otimes\psi(v_i,\cdot) \left(L+ \lambda I\right)^{-1}\psi(v_i,\cdot)\otimes\psi(v_i,\cdot) \left(L+ \lambda I\right)^{-\frac{1}{2}} \ , \\
& =& \frac{p(v_i)}{q(v_i)} \langle \psi(v_i,\cdot), \left(L+ \lambda I\right)^{-1}\psi(v_i,\cdot) \rangle_{L_2(P_x)} \frac{p(v_i)}{q(v_i)} \left(L+ \lambda I\right)^{-\frac{1}{2}}\psi(v_i,\cdot)\otimes\psi(v_i,\cdot) \left(L+ \lambda I\right)^{-\frac{1}{2}} \ ,\\
&=& \frac{\tau_{\lambda}(v_i)}{q(v_i)}\frac{p(v_i)}{q(v_i)}\left(L+ \lambda I\right)^{-\frac{1}{2}}\psi(v_i,\cdot)\otimes\psi(v_i,\cdot) \left(L+ \lambda I\right)^{-\frac{1}{2}} \ , \\
&\leq & d_{\tilde{\tau}}\frac{p(v_i)}{q(v_i)}\left(L+ \lambda I\right)^{-\frac{1}{2}}\psi(v_i,\cdot)\otimes\psi(v_i,\cdot) \left(L+ \lambda I\right)^{-\frac{1}{2}} \ .
\end{IEEEeqnarray*}

As a result, we have 
\begin{IEEEeqnarray*}{rCl}
\mathbb{E}\left( \mathbf{R}_i \mathbf{R}_i^*\right) &\preceq& d_{\tilde{\tau}} \mathbb{E}\left( \frac{p(v_i)}{q(v_i)}\left(L+ \lambda I\right)^{-\frac{1}{2}}\psi(v_i,\cdot)\otimes\psi(v_i,\cdot) \left(L+ \lambda I\right)^{-\frac{1}{2}}\right) \ , \\
&=&d_{\tilde{\tau}} \left(L+ \lambda I\right)^{-\frac{1}{2}} L\left(L+ \lambda I\right)^{-\frac{1}{2}} := \mathbf{M}\ .
\end{IEEEeqnarray*}

It is easy to see $m = \|\mathbf{M}\| \leq d_{\tilde{\tau}}$. We now let $\bar{\mathbf{R}}_s = \frac{1}{s}\sum_{i=1}^s \mathbf{R}_i$ and $d = \frac{\textnormal{Tr}(\mathbf{M})}{d_{\tilde{\tau}}} = d(\lambda)$, where we recall $d(\lambda) = \textnormal{Tr}\left(L(L+\lambda I)^{-1}\right)$. Applying the Bernstein inequality in Lemma~\ref{matx_con}, we have \[P\left(\left\| \bar{\mathbf{R}}_s - \mathbf{R}\right\| \geq \epsilon \right)\leq 4d(\lambda)\exp\left(\frac{-s\epsilon^2}{d_{\tilde{\tau}}+2d_{\tilde{\tau}}}\epsilon/3 \right).\]

By letting $4d(\lambda)\exp\left(\frac{-s\epsilon^2}{d_{\tilde{\tau}}+2d_{\tilde{\tau}}}\epsilon/3 \right) \leq \delta$ and rearrange the terms, we prove the lemma.
\end{proof}

\section{Proof of Theorem \ref{theo:minx}}

\thmworLipsch*
\begin{proof}
To this end, we now consider the estimator $\tilde{f}^{\lambda} \in \tilde{\mathcal{H}}$ obtained as $\tilde{f}^{\lambda}: = \argmin \mathbb{E}_n(l_{\tilde{f}})$ subject to $\|\beta\|_2^2 \leq 2R^2$ and recall $\tilde{f}_{\mathcal{H}} \in \tilde{\mathcal{H}}$, we have the following decomposition 
\begin{IEEEeqnarray*}{rCl}
\mathbb{E}(l_{\tilde{f}^{\lambda}}) &=& \mathbb{E}(l_{\tilde{f}^{\lambda}})- \mathbb{E}_n(l_{\tilde{f}^{\lambda}})+\mathbb{E}_n(l_{\tilde{f}^{\lambda}})-\mathbb{E}_n(l_{\tilde{f}_{\mathcal{H}}})+\mathbb{E}_n(l_{\tilde{f}_{\mathcal{H}}})-\mathbb{E}(l_{\tilde{f}_{\mathcal{H}}})+\mathbb{E}(l_{\tilde{f}_{\mathcal{H}}})-\mathbb{E}(l_{f_{\mathcal{H}}})+\mathbb{E}(l_{f_{\mathcal{H}}}) \ , \\
&\leq & \mathbb{E}(l_{\tilde{f}^{\lambda}})- \mathbb{E}_n(l_{\tilde{f}^{\lambda}}) +\mathbb{E}_n(l_{\tilde{f}_{\mathcal{H}}})-\mathbb{E}(l_{\tilde{f}_{\mathcal{H}}})+\mathbb{E}(l_{\tilde{f}_{\mathcal{H}}})-\mathbb{E}(l_{f_{\mathcal{H}}})+\mathbb{E}(l_{f_{\mathcal{H}}}) \ ,\\
&\leq & O\left(\frac{1}{\sqrt{n}}\right)  +\mathbb{E}(l_{\tilde{f}_{\mathcal{H}}})-\mathbb{E}(l_{f_{\mathcal{H}}})+\mathbb{E}(l_{f_{\mathcal{H}}}) \ ,\\
&\leq & O\left(\frac{1}{\sqrt{n}}\right)  + M \left\|\tilde{f}_{\mathcal{H}}-f_{\mathcal{H}} \right\|_{L_2(P_x)}+\mathbb{E}(l_{f_{\mathcal{H}}}) \ , \\
&\leq & 2MR\lambda^{r} + O\left(\frac{1}{\sqrt{n}}\right)  +  \mathbb{E}(l_{f_{\mathcal{H}}})\ .
\end{IEEEeqnarray*}
where for the second step, we use the fact that $\tilde{f}^{\lambda}$ is the minimizer of the empirical risk. For the third step, we use standard results for Rademacher complexities of $L_2$-balls \citep[Theorem 8]{bartlett2002rademacher} applied to $\mathbb{E}(l_{\tilde{f}^{\lambda}})- \mathbb{E}_n(l_{\tilde{f}^{\lambda}})$ and $\mathbb{E}_n(l_{\tilde{f}_{\mathcal{H}}})-\mathbb{E}(l_{\tilde{f}_{\mathcal{H}}})$. For the fourth step, we utilize the Lipschitz continuity of the loss function and the last step is due to Theorem~\ref{theo:f_H_approx_error}.
\end{proof}

\section{Proof of Theorem \ref{theo:fast_rate}}
\thmrefLipsch*
\begin{proof}
To prove Theorem \ref{theo:fast_rate}, we rely on the notion of local Rademacher complexity introduced in Lemma \ref{lma:local_rade_tr}. In order to do that, we need two steps. The first step is to find a proper sub-root function $\hat{\varphi}_n(r)$. The second step is to find the fixed point of $\hat{\varphi}_n(r)$. Hence, the following is devoted to solving these two problems. 

We first recall that $\tilde{\mathcal{H}}$ is an RKHS and hence is convex. Therefore, it is a star hull around every of its interior point (i.e., $\operatorname{star}\left(\tilde{\mathcal{H}}, f_{0}\right)=\left\{f_{0}+\alpha\left(f-f_{0}\right) \mid f \in \tilde{\mathcal{H}} \wedge \alpha \in[0,1]\right\} = \tilde{\mathcal{H}} , \forall f_0 \in \tilde{\mathcal{H}}$). In addition, let $f_{\tilde{\mathcal{H}}}$ be the estimator minimizing the expected risk in RKHS $\tilde{\mathcal{H}}$, subject to $\|f\|_{\tilde{\mathcal{H}}} \leq 2R^2$, Assumption A.$4$ then implies that $\mathbb{E}\left(f-f_{\tilde{\mathcal{H}}}\right)^2 \leq B \mathbb{E}\left(l_f - l_{f_{\tilde{\mathcal{H}}}}\right)$. We also recall $\tilde{f}^{\lambda} := \argmin_{f \in \tilde{\mathcal{H}}} \mathbb{E}_n(l_f)$ subject to $\|f\|_{\tilde{\mathcal{H}}} \leq 2R^2$.

Let $\psi(r)$ be a sub-root function and satisfy
\[\psi(r) \geq B L \mathbb{E} R_{n}\left\{f \in \tilde{\mathcal{H}}: L^{2} \mathbb{E}\left(f-f_{\tilde{\mathcal{H}}}\right)^{2} \leq r\right\}.\]
We apply Lemma~\ref{lma:local_rade_tr} to the class $l_{f}-l_{f_{\tilde{\mathcal{H}}}}$ with $T(f)=L^{2} \mathbb{E}\left(f-f_{\tilde{\mathcal{H}}}\right)^{2}$. By the contraction property of Rademacher complexity \citep[Theorem 14]{bartlett2002rademacher} and the symmetry of the Rademacher variables, $L \mathbb{E} R_{n}\left\{f: L^{2} \mathbb{E}\left(f-f_{\tilde{\mathcal{H}}}\right)^{2} \leq r\right\} \geq \mathbb{E} R_{n}\left\{l_{f}-l_{f_{\tilde{\mathcal{H}}}}: L^{2} \mathbb{E}\left(f-f_{\tilde{\mathcal{H}}}\right)^{2} \leq r\right\}$. Now by Lemma~\ref{lma:local_rade_tr} and noticing that $\mathbb{E}_{n}\left(l_{\tilde{f}^{\lambda} }-l_{f_{\tilde{\mathcal{H}}}}\right) \leq 0$, we have for any $\delta \in (0,1)$ and any $r \geq \psi(r)$, with probability at least $1-\delta$
\[\mathbb{E}\left(l_{\tilde{f}^{\lambda} }-l_{f_{\tilde{\mathcal{H}}}}\right) \leq  c_1 r+\frac{c_2}{n}\log \frac{1}{\delta}.\]
We now define \[\hat{\psi}(r)=c_{3} \hat{R}_{n}\left\{f \in \tilde{\mathcal{H}}: \mathbb{E}_{n}(f-\hat{f})^{2} \leq c_{4} r\right\}+\frac{c_{5}}{n}.\]
By \citet[Theorem 5.4]{bartlett2005local}, we have $\psi(r) \leq \hat{\psi}(r)$ and $r^* \leq \hat{r}^*$ where $r^*$ and $\hat{r}^*$ are the fixed points of $\psi(r)$ and $\hat{\psi}$ respectively. Therefore, with probability at least $1-2\delta$,
\begin{IEEEeqnarray}{rCl}
\mathbb{E}\left(l_{\tilde{f}^{\lambda} }-l_{f_{\tilde{\mathcal{H}}}}\right) \leq c_6 \hat{r}^{*}+\frac{c_7}{n}\log \frac{1}{\delta} \ . \label{eq:l_f_risk}
\end{IEEEeqnarray}

We are now left to compute the fixed point $\hat{r}^*$. To this end, we notice
\begin{align}
\begin{aligned}
&\hat{R}_n\{ f - f_{\tilde{\mathcal{H}}}, f\in \tilde{\mathcal{H}} \mid \mathbb{E}_n[f-f_{\tilde{\mathcal{H}}}]^2 \leq r \} \leq  &\\
&\hat{R}_n\{ f - g, f, g \in \tilde{\mathcal{H}} \mid \mathbb{E}_n[f-g]^2 \leq r \} =  &\\
&2\hat{R}_n\{ f , f\in \tilde{\mathcal{H}} \mid \mathbb{E}_n[f]^2 \leq r/4 \}  &
\end{aligned}\nonumber
\end{align}
We now let $\tilde{\mu}_1,\dots,\tilde{\mu}_s$ to be the eigenvalues of the normalized Gram-matrix $1/n\tilde{\mathbf{K}}$, by Lemma~\ref{lma:local_kernel}, we have \[2\hat{R}_n\{ f , f\in \tilde{\mathcal{H}} \mid \mathbb{E}_n[f]^2 \leq r/4 \} \leq 2 \left(\frac{2}{n}\sum_{i=1}^n\min\{r/4,\tilde{\mu}_i\}\right)^{1/2}.\]
For sub-root function $\hat{\psi}(r) = 2 \left(\frac{2}{n}\sum_{i=1}^n\min\{r/4,\tilde{\mu}_i\}\right)^{1/2} + \frac{c_5}{n}$, we first notice that adding some constant $c$ will only increase the fixed point by at most $2c$. Therefore, it suffice to compute the fixed point of 
\[\hat{r} \leq  2M \left(\frac{2}{n}\sum_{i=1}^n\min\{\hat{r}/4,\tilde{\mu}_i\}\right)^{1/2}.\]
To this end, we have 

\begin{IEEEeqnarray*}{rCl}
\left(\frac{\hat{r}}{2M}\right)^{2} & \leq& \frac{2}{n} \sum_{i=1}^{n} \min \left\{\frac{\hat{r}}{4}, \tilde{\mu}_{i}\right\} \ , \\
&=&\frac{2}{n} \min _{S \subseteq\{1, \ldots, n\}}\left(\sum_{i \in S} \frac{\hat{r}}{4}+\sum_{i \notin S} \tilde{\mu}_{i}\right)\ , \\
&=&\frac{2}{n} \min _{0 \leq h \leq n}\left(\frac{ h \hat{r}}{4}+\sum_{i>h} \tilde{\mu}_{i}\right) \ .
\end{IEEEeqnarray*}

Finally, we solve the quadratic inequality for each value of $h$ implies
\[\hat{r} \leq c \min _{0 \leq h \leq n}\left(\frac{h}{n}+\sqrt{\frac{1}{n} \sum_{i>h} \tilde{\mu}_{i}}\right).\]


Recall the estimator $\tilde{f}^{\lambda} \in \tilde{\mathcal{H}}$ is obtained as $\tilde{f}^{\lambda}: = \argmin \mathbb{E}_n(l_{\tilde{f}})$ subject to $\|\beta\|_2^2 \leq 2R^2$ and recall $f_{\tilde{\mathcal{H}}} \in \tilde{\mathcal{H}}$.
By Eq.~\ref{eq:l_f_risk}, we have 
\begin{IEEEeqnarray*}{rCl}
\mathbb{E}(l_{\tilde{f}^{\lambda}}-l_{f_{\tilde{\mathcal{H}}}}) &\leq & c_6\hat{r}^* + \frac{c_7}{n}\log \frac{1}{\delta} \ ,
\end{IEEEeqnarray*}
Equivalently, we can rewrite the above equation as
\begin{IEEEeqnarray*}{rCl}
\mathbb{E}(l_{\tilde{f}^{\lambda}})-\mathbb{E}(l_{f_{\mathcal{H}}})&\leq& \mathbb{E}(l_{f_{\tilde{\mathcal{H}}}}) - \mathbb{E}(l_{f_{\mathcal{H}}}) + c_1\hat{r}^* + \frac{c_2}{n}\log \frac{1}{\delta} \ , \\
&\leq & \mathbb{E}(l_{\tilde{f}_{\mathcal{H}}}) - \mathbb{E}(l_{f_{\mathcal{H}}}) + c_1\hat{r}^* + \frac{c_2}{n}\log \frac{1}{\delta} \ , \\
&\leq & M \left\|\tilde{f}_{\mathcal{H}} - f_{\mathcal{H}} \right\|_{L_2(P_x)}+ c_1\hat{r}^* + \frac{c_2}{n}\log \frac{1}{\delta} \ , \\
&\leq & 2MR\lambda^{r} + c_1\hat{r}^* + \frac{c_2}{n}\log \frac{1}{\delta} \ ,
\end{IEEEeqnarray*}
where the second step is because $\mathbb{E}\left( l_{f_{\tilde{\mathcal{H}}}}\right) \leq \mathbb{E}\left(l_{\tilde{f}_{\mathcal{H}}}\right)$. 
\end{proof}

\section{Matrix Bernstein Inequality}
The next lemma is the matrix Bernstein inequality, cited from \citet[Lemma 27]{avron2017random} which is a restatement of Corollary 7.3.3 in \citet{tropp2015introduction} with some fix in the typos.
\begin{lma}\citep[Bernstein inequality,][Corollary 7.3.3]{tropp2015introduction}\label{matx_con}
Let $\mathbf{R}$ be a fixed $d_1 \times d_2$ matrix over the set of complex/real numbers. Suppose that $\{\mathbf{R}_1,\cdots,\mathbf{R}_n\}$ is an independent and identically distributed sample of $d_1 \times d_2$ matrices such that \[\mathbb{E}[\mathbf{R}_i] = \mathbf{R} \qquad \text{and} \qquad \|\mathbf{R}_i\|_2 \leq L\ ,\]
where $L>0$ is a constant independent of the sample.
Furthermore, let $\mathbf{M}_1, \mathbf{M}_2$ be semidefinite upper bounds for the matrix-valued variances 
\begin{align*}
\begin{aligned}
& \mathrm{Var}_1[\mathbf{R}_i] \preceq \mathbb{E}[\mathbf{R}_i\mathbf{R}_i^{T}] \preceq \mathbf{M}_1 & \\
& \mathrm{Var}_2[\mathbf{R}_i] \preceq \mathbb{E}[\mathbf{R}_i^{T}\mathbf{R}_i]\preceq \mathbf{M}_2\ . &
\end{aligned}
\end{align*}
Let $m = \max(\|\mathbf{M}_1\|_2,\|\mathbf{M}_2\|_2)$ and $d =\frac{\text{Tr}(\mathbf{M}_1)+ \text{Tr}(\mathbf{M}_2)}{m}.$ 
Then, for $\epsilon \geq \sqrt{m/n}+2L/3n$, we can bound \[\bar{\mathbf{R}}_n = \frac{1}{n}\sum_{i=1}^{n}\mathbf{R}_i\] around its mean using the concentration inequality \[P(\|\bar{\mathbf{R}}_n - \mathbf{R}\|_2 \geq \epsilon) \leq 4d\exp\Bigg(\frac{-n\epsilon^2/2}{m+2L\epsilon/3}\Bigg)\ .\]
\end{lma}

\section{Local Rademacher Complexities}
In the refined case, our analysis relies on the local Rademacher complexities technique developed by \citet{bartlett2005local}. As shall be seen later, local Rademacher complexity is often linked with the so-called sub-root function which we define below.

\begin{defn}\label{apn: sub-root}
Let $\varphi: [0,\infty) \rightarrow [0,\infty)$ be a function. Then, $\varphi(r)$ is called a \textit{sub-root} function if, for all $r>0$, $\varphi(r)$ is non-decreasing and $\frac{\varphi(r)}{r}$ is non-increasing.
\end{defn}

A sub-root function has the following property.

\begin{lma}\citep[Lemma 3.2]{bartlett2005local} \label{apn:sub-pro}
If $\varphi(r)$ is a sub-root function, then $\varphi(r) = r$ has a unique positive solution $r^*$. In addition, we have that $r \geq \varphi(r) $ if and only if $r \geq r^*$.
\end{lma}

After we introduce the Rademacher complexity and the sub-root function, we are now ready to state the theorem on local Rademacher complexity we need in the refined analysis.

\begin{lma}\citep[Theorem 3.3]{bartlett2005local}\label{lma:local_rade_tr}
Let $\mathcal{F}$ be a class of functions with bounded ranges and assume that there are some functional $T: \mathcal{F} \rightarrow \mathbb{R}^{+}$ and some constant $B$ such that for every $f \in \mathcal{F}, \textnormal{Var}[f] \leq T(f) \leq B P f .$ Let $\psi$ be a sub-root function and $r^{*}$ be the fixed point.
Assume that $\psi$ satisfies, for any $r \geq r^{*}$
\[\psi(r) \geq B R_{n}\{f \in \mathcal{F}: T(f) \leq r\}.\]
Then, for any $f \in \mathcal{F}$, some constant $D>1$ and every $\delta \in (0,1)$, with probability at least $1-\delta$
\[\mathbb{E}(f) \leq \frac{D}{D-1} \mathbb{E}_n(f) +\frac{c_{1} D}{B} r^{*}+\frac{c_2}{n} \log \frac{1}{\delta},\] where $c_1,c_2$ are some universal constants.
\end{lma}

In addition to the above theorem, one also need to quantitatively characterize the local Rademacher complexity in order to obtain a tight upper bound. Fortunately in the kernel case, we are able to do so as illustrate in the following lemma.
\begin{lma}\citep[Lemma 6.6]{bartlett2005local}\label{lma:local_kernel}
Let $k$ be a positive definite kernel function with reproducing kernel Hilbert space $\mathcal{H}$ and let $\hat{\lambda}_1 \geq \cdots \geq \hat{\lambda}_n$ be the eigenvalues of the normalized Gram-matrix $(1/n)\mathbf{K}$. Then, for all $r >0$ and $f \in \mathcal{H}$, \[\hat{R}_n\{f\in \mathcal{H} \mid \mathbb{E}_n[f^2] \leq r\} \leq \left(\frac{2}{n}\sum_{i=1}^n\min\{r,\hat{\lambda}_i\}\right)^{1/2}.\]
\end{lma}
\end{document}